\documentclass{article}
\usepackage{graphicx} % Required for inserting images
\usepackage{amsmath}
\usepackage{amsthm}
\usepackage{amssymb}
\usepackage{amsfonts}
\usepackage{authblk}
\usepackage{hyperref}
\usepackage{lettrine}
\usepackage{tikz,tikz-cd,tikz-qtree}
\usepackage{quiver}
\usepackage{turnstile}
\usepackage{pifont}
\usepackage[backend=biber, sortcites, uniquename=init, giveninits]{biblatex}
\newcommand{\cmark}{\ding{51}}%
\newcommand{\xmark}{\ding{55}}%

\theoremstyle{definition}
\newtheorem{defn}{Definition}[section]

\theoremstyle{plain}
\newtheorem{prop}{Proposition}[section]

\makeatletter
\newcommand*{\da@rightarrow}{\mathchar"0\hexnumber@\symAMSa 4B }
\newcommand*{\da@leftarrow}{\mathchar"0\hexnumber@\symAMSa 4C }
\newcommand*{\xdashrightarrow}[2][]{%
  \mathrel{%
    \mathpalette{\da@xarrow{#1}{#2}{}\da@rightarrow{\,}{}}{}%
  }%
}
\newcommand{\xdashleftarrow}[2][]{%
  \mathrel{%
    \mathpalette{\da@xarrow{#1}{#2}\da@leftarrow{}{}{\,}}{}%
  }%
}
\newcommand*{\da@xarrow}[7]{%
  \sbox0{$\ifx#7\scriptstyle\scriptscriptstyle\else\scriptstyle\fi#5#1#6\m@th$}%
  \sbox2{$\ifx#7\scriptstyle\scriptscriptstyle\else\scriptstyle\fi#5#2#6\m@th$}%
  \sbox4{$#7\dabar@\m@th$}%
  \dimen@=\wd0 %
  \ifdim\wd2 >\dimen@
    \dimen@=\wd2 %   
  \fi
  \count@=2 %
  \def\da@bars{\dabar@\dabar@}%
  \@whiledim\count@\wd4<\dimen@\do{%
    \advance\count@\@ne
    \expandafter\def\expandafter\da@bars\expandafter{%
      \da@bars
      \dabar@ 
    }%
  }%  
  \mathrel{#3}%
  \mathrel{%   
  \mathop{\da@bars}\limits
  \ifx\\#1\\%
  \else
  _{\copy0}%
  \fi
  \ifx\\#2\\%
  \else
  ^{\copy2}%
  \fi
  }%   
  \mathrel{#4}%
  }
  \makeatother

\title{da Costa and Tarski meet Goguen and Carnap:\\
\large a novel approach for ontological heterogeneity based on consequence systems}
\author{
Gabriel Rocha\\
\small Centre for Logic, Epistemology and the History of Science\\
\small Universidade Estadual de Campinas (Unicamp) \\
\small \texttt{gabrielrocha.comp [at] gmail.com}}
\date{}

\begin{document}

\maketitle

\begin{abstract}
    This paper presents a novel approach for ontological heterogeneity that draws heavily from Carnapian-Goguenism, as presented by \textcite{Kutz_Mossakowski_Lücke_2010}. The approach is provisionally designated da Costian-Tarskianism, named after da Costa's Principle of Tolerance in Mathematics and after Alfred Tarski's work on the concept of a consequence operator. The approach is based on the machinery of consequence systems, as developed by \textcite{carnielli2008analysis} and \textcite{citkin2022consequence}, and it introduces the idea of an extended consequence system, which is a consequence system extended with ontological axioms. The paper also defines the concept of an extended development graph, which is a graph structure that allows ontologies to be related via morphisms of extended consequence systems, and additionally via other operations such as fibring and splitting. Finally, we  discuss the implications of this approach for the field of applied ontology and suggest directions for future research.
\end{abstract}

\section{Introduction}

It is not that there exists high diversity in all facets of what constitutes an \textbf{o}ntology in the context of applied ontology\footnote{In this paper, the term 'ontology' is written in two different ways to minimize confusion: ontology to represent the abstract concept of an ontology, and \textbf{o}ntology to represent a concrete instance of an ontology.}. Futhermore, to reason over more than a single \textbf{o}ntology is often desirable or necessary, leading to the concepts of \textit{ontological heterogeneity} and \textit{interoperability} \cite{visser1997analysis}. Ontological heterogeneity may be defined as a scenario in which two \textbf{o}ntologies make different ontological assumptions about the same domain, potentially using different representation languages. Furthermore, ontological interoperability refers to the ability to reason over \textbf{o}ntologies in a heterogeneous scenario. Heterogeneity and interoperability are not limited to computational systems, such as Ontology-Based Data Access (OBDA) systems, however the concepts are frequently seen as causes for complexity in this context.

Introducing heterogeneity may difficult interoperability, thereby leading the need of develop methods and tools to manage it. The approach presented in this paper is a method to manage ontological heterogeneity, and it is based on the machinery of consequence systems, as developed by \textcite{carnielli2008analysis} and \textcite{citkin2022consequence}, and on the meta-logical principles of what is tentativeley called da Costa-Tarskianism, as presented by \textcite{marcos2004possible}. Our approach is inspired by the Carnapian-Goguenist enterprise, as presented by \textcite{Kutz_Mossakowski_Lücke_2010}, and ours can be understood as a dual to it. The key underlying difference between the two approaches is that da Costian-Tarskianism utilizes consequence systems to represent classes of logics, whereas Carnapian-Goguenism utilizes institutions. This difference has implications for the way \textbf{o}ntologies are defined and related, as will be shown in the following sections. For further references on the topic of ontological diversity, and how it relates to heterogeneity and interoperability, see the surveys by \textcite{10.2478/cait-2025-0035} and \textcite{schneider-simkus-2020}.

\section{What is a logic?}

Unlike Carnapian-Goguenism, the da Costian-Tarskianist approach does not utilize institutions to define classes of logics via their \textit{semantics}. Instead, consequence systems are chosen to represent classes of logics via their \textit{syntax}. As \textcite{carnielli2008analysis} have shown, various kinds of calculi, such as Hilbert calculi, sequent calculi and tableau calculi, induce consequence systems. Consequently, consequence systems are sufficiently abstract to permit the construction of heterogeneous operations.

\textcite{tarski1930fundamentale} initially presented the consequence operator as way to characterize what a logical system should preserve in order to be called a proper logic. The operator draws from the closure operator originally presented by Kazimierz Kuratowski, and it attempts to capture, as suggested by its name, the notion of what formulas should follow as syntactical consequences of a given theory in a rather general and abstract way. The consequence operator is defined as follows, as per \textcite[p. 17]{marcos2004possible}:

\begin{defn}[Consequence operator or relation]
    For a given set $For$ of formulas, $\dststile{}{} \subseteq \mathcal{P}(For) \times For$ is a consequence operator (or relation) if the following three conditions hold:
    \begin{enumerate}
        \item If $A \in \Gamma$ then $\Gamma \dststile{}{} A$ \hfill reflexivity
        \item If $\Delta \dststile{}{} A$ and $\Delta \subseteq \Gamma$ then $\Gamma \dststile{}{} A$ \hfill monotonicity
        \item If $\Delta \dststile{}{} A$ and $\Gamma, A \dststile{}{} B$ then $\Delta, \Gamma \dststile{}{} B$ \hfill transitivity
    \end{enumerate}
\end{defn}

\textcite{marcos2004possible} renamed da Costa's Principle of Tolerance in Mathematics to Principle of Non-Triviality, and they formalized it using the consequence operator $\dststile{\mathbf{L}}{}$ over a given logic $\mathbf{L}$. In addition to the Principle of Non-Triviality, they formalized two other logical principles:

\textsc{Principle of Non-Contradiction} (PNC)\\
$\mathbf{L}$ must be non-contradictory: $\exists \Gamma \forall A~ (\Gamma \not \dststile{\mathbf{L}}{} A \text{ or } \Gamma \not \dststile{\mathbf{L}}{} \lnot A)$

\textsc{Principle of Non-Triviality} (PNT)\\
$\mathbf{L}$ must be non-trivial: $\exists \Gamma \exists B~ (\Gamma \not \dststile{\mathbf{L}}{} B)$

\textsc{Principle of Explosion} (PoE) or \textsc{Principle of Pseudo-scotus} (PPS)\\
$\mathbf{L}$ must be explosive: $\forall \Gamma, A, B~ (\Gamma, A, \lnot A \dststile{\mathbf{L}}{} B)$

The objective of defining such principles is to capture meta-logical properties of a given logic. In this regard, da Costa's principle is in essence stating that no admissible logic should violate PNT. \textcite{marcos2004possible} defines, for instance, a paraconsistent logic as a logic such that PNC and PNT hold but not PPS. This meta-logical framework consists of the foundations of da Costa-Tarskianism.

Prior to defining consequence systems, it is necessary to present some preliminaries. Many of the definitions hereafter are due to \textcite{carnielli2008analysis}, the first of which is that of a signature.

\begin{defn}[Signature]
    A signature $C$ is a countable family of sets $C_k$ where $k \in \mathbb{N}$. The index $k$ states the arity of the connectives, meaning $C_k$ is the set of $k$-ary connectives. Usually $C$ is a finite set.
\end{defn}

The set $C_0$ of zero-ary connectives over a signature $C$ is usually called the set of \textit{constants} of $C$. Given two signatures $C'$ and $C''$, it is customary to write $C' \leq C''$ if $C''_k \subseteq C'_k$ for every $k \in \mathbb{N}$. It may be noted that a signature $C$, in the current context, \textit{does} contain the logical symbols.

A signature induces a language over it. Intuitively, the language of the signature is the set of all formulas that can be inductively constructed using its connectives and variables.

\begin{defn}[Language of a signature]
    Let $C$ be a signature and $\Xi = \{ \xi_n : n \in \mathbb{N}^+ \}$ be an enumerable set called the set of \textit{schema variables}. The language over $C$ is the set $L(C)$ inductively defined as follows:
    \begin{itemize}
        \item $\xi \in L(C)$ for every $\xi \in \Xi$
        \item $c \in L(C)$ for every $c \in C_0$
        \item $(c(\phi_1, \ldots, \phi_k)) \in L(C)$ whenever $c \in C_k$, $k \geq 1$, and $\phi_1, \ldots, \phi_k \in L(C)$.
    \end{itemize}
\end{defn}

Signatures can be related in a categorial setting via signature morphisms. Signatures and signature morphisms constitute the category $\mathbf{Sig}$. Observe that in this context, $\mathbf{Sig}$ denotes a \textit{specific} and highly abstract category of signatures tailored for the approach, whereas the category $\mathbf{Sign}$ in the definition of an institution is specific to a particular institution.

\begin{defn}[Signature morphism]
    A \textit{signature morphism} $h : C \rightarrow C'$ is a family of maps $h_k : C_k \rightarrow C_k'$, where $k \in \mathbb{N}$.
\end{defn}

\begin{defn}[Category of signatures]
        The category of signatures $\mathbf{Sig}$ is defined as follows:
    \begin{itemize}
        \item The objects are signatures $C$
        \item The morphisms are signature morphisms as described above
    \end{itemize}
    Composition of morphisms $h$ and $g$ is defined as the composition of each indexed map, i.e. $h \circ g$ is the family of maps $h_k \circ g_k$ and the identity morphism is the family map of identity morphisms, i.e. $id_k$, where $k \in \mathbb{N}$.
\end{defn}

One may generate a new signature from two inputs by means of fibring. The fibring of two signatures is defined below.

\begin{defn}
    The \textit{fibring of signatures} $C'$ and $C''$ is the signature
    $$ C' \cup C''$$
    such that $(C' \cup C'')_k = C'_k \cup C''_k$ for every $k \in \mathbb{N}$.
\end{defn}

Signatures and their associated languages are fundamental to define consequence systems. As \textcite{citkin2022consequence} note, \citeauthor{tarski1930fundamentale} not only presented consequence operators in his 1930 paper but also consequence systems indirectly. Intuitively, consequence systems talk about the closure of consequence operators --- that is, the set of all formulas which follow from a given theory. Consequence systems and consequence operators induce one another, as pointed out by \textcite{citkin2022consequence}.

\begin{defn}[Consequence System]
    \label{def-cons-system}
    A \textit{consequence system} is a pair $\langle C, \mathsf{C} \rangle$ where $C$ is a signature and $\mathsf{C} : \mathcal{P}L(C) \rightarrow \mathcal{P}L(C)$ is a map with the following properties:
    \begin{enumerate}
        \item $\Gamma \subseteq \mathsf{C}(\Gamma)$ \hfill extensivity
        \item $\Gamma_1 \subseteq \Gamma_2$ then $\mathsf{C}(\Gamma_1) \subseteq \mathsf{C}(\Gamma_2)$ \hfill monotonicity
        \item $\mathsf{C}(\mathsf{C}(\Gamma)) \subseteq \mathsf{C}(\Gamma)$ \hfill idempotence
    \end{enumerate}
\end{defn}

Before proceeding, it is necessary examine in greater detail the reasons for choosing consequence systems over more complex machinery from proof theory, such as abstract proof systems as presented by \textcite{carnielli2008analysis} themselves or even a categorial approach as described by \textcite{HYLAND200243}. As previously noted, consequence systems can be employed to represent the syntactical aspect of a wide array of logics. Indeed, consequence systems are sufficiently capable of expressing, intuitively, what set of sentences may be deduced from another set of sentences, but not precisely how. That is, the consequence operator does not encode a proof: the entire sequence of steps necessary to deduce a set of consequents from a set of antecedents. While this may be a shortcoming for certain applications where explainability is a major concern, it is not the case from the ontological point of view. The primary objective of a heterogeneous applied ontological framework is to provide the ability to describe and reason about concepts. However, the specific mechanisms underlying reasoning are relegated to the domains of metaphysics and grounding. The primary contention of this paper is that, in the field of applied ontology, the \textbf{o}ntologies do not require self-awareness, that is, they do not need to encode information about their own inner workings.

Consequence systems represent a \textit{class} of logics and are devoid of a semantic construct, thus exhibiting a sort of duality to institutions. A consequence system may also be attached to a semantic construct, such as algebraic semantics as done via \citeauthor{carnielli2008analysis}'s interpretation systems \cite[p. 92]{carnielli2008analysis}, resulting in a signature-\textbf{bound} logic (also called a logic system). Because consequence systems are built on top of a \textit{particular} signature, this is a key distinction between them and institutions. This fact is also the reason the duality is not mirror-like. As noted by \textcite{Coniglio2005-CONTAS}, a better candidate for proper mirror-like duality, when compared to institutions, are \citeauthor{meseguer1989general}'s entailment systems.

% \marginpar{To write examples.}

\textcite{carnielli2008analysis} also define some particular types of consequence systems. In what follows, $\mathcal{P}_\mathsf{fin} S$ denotes the set of all finite subsets of $S$.

\begin{defn}[Compact Consequence System]
    A consequence system $\langle C, \mathsf{C} \rangle$ is \textit{compact} or \textit{finitary} if
    $$ \mathsf{C}(\Gamma) = \bigcup_{\Phi \in \mathcal{P}_{\mathsf{fin}}\Gamma} \mathsf{C}(\Phi) $$
    For each $\Gamma \subseteq L(C)$. Compact consequence systems are also called \textit{abstract systems} in \cite{davey_priestley_2002}.
\end{defn}

\begin{defn}[Quasi-consequence System]
    A \textit{quasi-consequence system} is a consequence system such that idempotence, as defined in \ref{def-cons-system} does not necessarily hold.
\end{defn}

\begin{defn}
    A consequence system $\langle C, \mathsf{C} \rangle$ is \textit{closed for renaming substitutions} if
    $$ \sigma(\mathsf{C}(\Gamma)) \subseteq \mathsf{C}(\sigma(\Gamma))$$
    for every $\Gamma \in L(C)$ and every renaming substitution $\sigma$, i.e. $\sigma(\{\zeta\}) = \{\zeta'\}$ for each $\zeta \in \Xi$, where $\zeta' \in \Xi$. If the inclusion holds for every substitution, the consequence system is called \textit{structural}.
\end{defn}

Consequence systems may be related in an order theory sense, by introducing a weakness relation, and also in a categorial setting, by defining morphisms between consequence systems. Consequence systems and their morphisms constitute a category, $\mathbf{Csy}$.

\begin{defn}[Weakness relation]
    \label{cons-weakness-relation}
    The consequence system $\langle C, \mathsf{C} \rangle$ is \textit{weaker} than consequence system $\langle C', \mathsf{C}' \rangle$, written
    $$ \langle C, \mathsf{C} \rangle \leq \langle C', \mathsf{C}' \rangle $$
    if $L(C) \subseteq L(C')$ and $\mathsf{C}(\Gamma) \subseteq \mathsf{C}'(\Gamma)$ for every subset $\Gamma$ of $L(C)$.
    Additionally, $\langle C, \mathsf{C} \rangle$ is \textit{partially weaker} than consequence system $\langle C', \mathsf{C}' \rangle$, written
    $$ \langle C, \mathsf{C} \rangle \leq_p \langle C', \mathsf{C}' \rangle $$
    if $L(C) \subseteq L(C')$ and $\mathsf{C}(\emptyset) \subseteq \mathsf{C}'(\emptyset)$.
\end{defn}

\begin{defn}[Consequence system morphisms]
    A \textit{consequence system morphism} $h : \langle C, \mathsf{C} \rangle \rightarrow \langle C', \mathsf{C}' \rangle$ is a map $h : L(C) \rightarrow L(C')$ such that
    $$ h(\mathsf{C}(\Gamma)) \subseteq \mathsf{C}'(h(\Gamma))$$
    for every $\Gamma \subseteq L(C)$.
\end{defn}

One can think of consequence system morphisms as transformations that weakly translate consequence. If $\{\phi\}$ is a consequent of $\Gamma$ via $\mathsf{C}$ in the source consequence system, then $h(\{\phi\})$ is a consequent of $h(\Gamma)$ in the target consequence system. 

\begin{defn}[Category of consequence systems]
    The category of consequence systems $\mathbf{Csy}$ is defined as follows:
    \begin{itemize}
        \item The objects are consequence systems $\langle C, \mathsf{C} \rangle$
        \item The morphisms are consequence system morphisms as described above
    \end{itemize}
    Composition of morphisms $h$ and $g$ is defined as the usual composition of maps $g \circ h$. The identity morphism is the identity map, i.e. $h(X) = X$ for every $X$.
\end{defn}

The category \textbf{Csy} may be thought of as the web which connects classes of logics (i.e. consequence systems) whenever it is possible to establish a weak translation between their deductive aspect. That is, if a consequence system's consequence map $\mathsf{C}$ is a subset of another consequence system, under some translation, then there is a morphism between them. 

\section{What is an ontology?}

Similarly to institutions, consequence systems, in and of themselves, are insufficient to describe \textbf{o}ntologies, as they lack the mechanisms to encode ontological information\footnote{This claim relies on the philosophical distinction between Ontology and Logic.}. For this reason, it is necessary to employ an approach analogous to the Carnapian-Goguenist way, whereby the logical machinery is extended with ontological axioms.

The intuitive understanding behind this approach is that an \textbf{o}ntology is merely a consequence system plus a theory (i.e. the set of ontological axioms). This method of definining \textbf{o}ntologies is more straight-forward than the Carnapian-Goguenist approach, although it initially places \textbf{o}ntologies in an \textbf{o}ntological vacuum as, by design, \textbf{o}ntologies do not encode information about their own inner workings or about other \textbf{o}ntologies. Extended development graphs represent a potential solution to this apparent limitation, as they permit \textbf{o}ntologies to be related.

It is necessary to note that one could ``append'' axioms, representing an ontological theory, to an existing consequence system by changing its consequence operator. By doing so, one obtains a new class of logics including the ontological theory as axioms. Thus, a consequence system could be defined as an \textbf{o}ntology itself. However, the reasons for not doing so are of philosophical and computational nature. On the first point, by making the ontological theory explicit, one draws the boundary between what is inherent to an intuitive ``mode of reasoning'' (i.e. a class of logics, the consequence system itself) and what is particular to a description of the world (i.e. the ontological axioms themselves). This position follows \citeauthor{daCosta2002-DACLAO}'s view on the distinction between Logic and Ontology, both as philosophical areas and as objects, as exposed in \cite{daCosta2002-DACLAO}. Furthermore, such position is aligned with \citeauthor{10.1006/ijhc.1996.0091}'s thesis that domain and reasoning knowledge should be independent \textcite{10.1006/ijhc.1996.0091}. With regard to the second point, clearly specifying the ontological axioms, and presenting an algorithmic way to manage such axioms, increases the inherent modularity of the approach. By not equating a consequence system to an \textbf{o}ntology, it is possible to utilize a single consequence systems to generate several \textbf{o}ntologies based on it.

As a preliminary step, below is the definition of an extended consequence system.

\begin{defn}[extended consequence system]
    An \textit{extended consequence system} is a quadruple $\langle C, \mathsf{C}, C_o, \Gamma_o \rangle$, where $C, C_o$ are signatures, $\mathsf{C}$ is a consequence map and $\Gamma_o \in L(C)$ is a set, such that:
    \begin{enumerate}
        \item $\langle C, \mathsf{C} \rangle$ is a consequence system
        \item For every $C_k \in C$ and $C_k' \in C_o$, $C_k' \subset C_k$
        \item For every $\phi \in \Gamma_o$, $\phi \in \mathsf{C}(\emptyset)$ --- i.e. $\Gamma_o$ is an axiomatic theory
    \end{enumerate}
\end{defn}

An \textbf{o}ntology is then defined as a \textit{particular kind} of extended consequence system. 

\begin{defn}[\textbf{O}ntology]
    An \textit{\textbf{o}ntology} is defined as a particular extended consequence system $\langle C, \mathsf{C}, C_o, \Gamma_o \rangle$. When the underlying $\langle C, \mathsf{C} \rangle$ consequence system is implicitly understood, we may drop it and refer to an \textbf{o}ntology by its components $C_o$ and $\Gamma_o$ (also called ontological aspect components).
\end{defn}

The rationale behind extending a consequence system is now evident. It is possible to distinguish between the ontological content and the purely logical machinery, while allowing the machinery to handle ontological knowledge. From this point onward, the terms ``extended consequence system'' and \textbf{o}ntology are used interchangeably in the context of the da Costa-Tarski approach. The point is that \textbf{o}ntologies are a specific kind of extended consequence systems whose theory $\Gamma$ is defined to match the definition of \textbf{o}ntology discussed and proposed by \textcite{guarino1998formal}.

Extended consequence systems and their morphisms constitute the category $\mathbf{ECsy}$:

\begin{defn}[Category of extended consequence systems (or \textbf{o}ntologies)]
    The category $\mathbf{ECsy}$ of extended consequence systems is defined as follows:
    \begin{itemize}
        \item Its objects are extended consequence systems $\langle C, \mathsf{C}, C_o, \Gamma_o \rangle$
        \item A morphism between extended consequence systems $\langle C^A, \mathsf{C}^A, C^A_o, \Gamma^A \rangle$ and $\langle C^B, \mathsf{C}^B, C^B_o, \Gamma^B \rangle$ is a consequence system morphism $h : \langle C^A, \mathsf{C}^A \rangle \rightarrow \langle C^B, \mathsf{C}^B \rangle$ such that $h(\Gamma^A) = \Gamma^B$.
    \end{itemize}
    Composition between extended consequence system morphisms works as expected, composing consequence system morphisms and composing the $\Gamma$ mappings. The identity morphism is the consequence system identity morphism, since $id(\Gamma) = \Gamma$.
\end{defn}

It should be noted that \textbf{o}ntologies are not defined with respect to development graphs and do not require a Grothendieck construction in order to relate two \textbf{o}ntologies based on different underlying logics. Nevertheless, it is possible to construct a graph structure that allows \textbf{o}ntologies to be related via morphisms of $\mathbf{ECsy}$, and additionally via other operations such as fibring and splitting. Consequently, \textbf{o}ntologies in the da Costa-Tarski sense are sufficiently flexible to be independent of external structures such as graphs. However, they permit the existence of such external structures in a manner that is both useful and coherent.

\section{How to relate ontologies?}

This section details the structures and mechanisms required to relate \textbf{o}ntologies. The following is based on the machinery developed by \textcite{carnielli2008analysis} to reinterpret morphisms and other operations between consequence systems as means to refine, integrate and connect \textbf{o}ntologies.

Once again inspired by the approach from \textcite{Kutz_Mossakowski_Lücke_2010}, we define the concept of an extended development graph. This will form the basis of our framework and will slowly become mutated with add-ons to increase complexity and expressivity of operations.

\begin{defn}
    An \textit{extended development graph} is a vertex and edge-labeled directed, acyclic graph (DAG) $\mathcal{AG} = \langle \mathcal{N}, \mathcal{L} \rangle$, where:
    \begin{itemize}
        \item $\mathcal{N}$ is a set of nodes, where each node $N \in \mathcal{N}$ is labeled with an extended consequence system $\langle C^N, \mathsf{C}^N, C^N_o, \Gamma^N_o \rangle$
        \item $\mathcal{L}$ is a sorted set of directed links:
        \begin{itemize}
            \item \textit{definition links} $K \xrightarrow{h} N$ where $h$ is an extended consequence system morphism
            \item \textit{theorem links} $K \xdashrightarrow{t} N$ for each $K, N \in \mathcal{N}$ if the consequence system labeling $K$ is weaker than the consequence system labeling $N$, c.f. definition \ref{cons-weakness-relation}.
        \end{itemize}
    \end{itemize}
\end{defn}

An extended development graph captures a web of \textbf{o}ntologies related by definition links, corresponding to morphisms, and induced links, corresponding to existing weakness relations. If there is a theorem link between two \textbf{o}ntologies, it is possible to informally assess that one is a non-conservative sub-\textbf{o}ntology of the other. Same as in the Carnapian-Goguenist approach, the fact extended development graphs are directed and acyclic ensures that there are no circular definitions.

\subsection{Refinement}

\textbf{O}ntologies defined via extended consequence systems can also be refined in a manner that is highly analogous to the process of refining their institution-based counterpart. There are two kinds of refinements, homogeneous and heterogeneous.

% \marginpar{When do refinements compose?}

\begin{defn}[Homogeneous refinements]
    Given two \textbf{o}ntologies $O_1$ and $O_2$ in an extended development graph, $O_2$ is called a \textit{homogeneous refinement} of $O_1$ if there is a theorem link $O_1 \xdashrightarrow{t} O_2$.
\end{defn}

\begin{defn}[Heterogeneous refinements]
    Given two \textbf{o}ntologies $O_1$ and $O_2$, $O_2$ is a \textit{heterogeneous refinement} of $O_1$ in the underlying extended development graph if:
    \begin{itemize}
        \item there is a third \textbf{o}ntology $O_2'$ such that a definition link $O_1 \xrightarrow{h} O_2$ exists, where $h$ is monomorphic
        \item there is a theorem link $O_1 \xdashrightarrow{t} O_2$
    \end{itemize}
\end{defn}

Refinements in the context of extended consequence systems are sufficiently similar to those of institutions that they may be represented diagrammatically in the same manner, as depicted in figure \ref{fig:het-refinement-aug}. However, the meaning of each link is quite different in each approach. In the Carnapian-Goguenist approach, for one \textbf{o}ntology to be heterogeneously refined into another, the latter must conserve models of the former. On the other hand, in the da Costian-Tarskian approach, refinement conserves theoremhood. 

    \begin{figure}[ht!]
        \centering
\[\begin{tikzcd}
	& {O_2'} \\
	{O_1} && {O_2}
	\arrow["t", dashed, from=2-1, to=1-2]
	\arrow["mono", "h"', from=2-3, to=1-2]
\end{tikzcd}\]
        \caption{Heterogeneous refinement from $O_1$ to $O_2$.}
        \label{fig:het-refinement-aug}
    \end{figure}
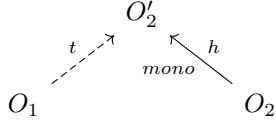

One can also define mirrored definitions of equivalence relations, as presented in definitions by \textcite{Kutz_Mossakowski_Lücke_2010} in section 3.1.2. Furthermore, two \textbf{o}ntologies are weakly equivalent if they can be heterogeneously and conservatively refined into each other (i.e. they preserve each other's theorems), and they are pre-equivalent if they share a common monomorphic extension.

\subsection{Integration}

As defined by \textcite{Kutz_Mossakowski_Lücke_2010}, integration between \textbf{o}ntologies is based on refinements. This approach will be followed here, as it is a parallel to the existing framework.
% \marginpar{Other kinds of integration?}

\begin{defn}[Heterogeneous integration]
    Let $O_1$ and $O_2$ be \textbf{o}ntologies and $O$ a so-called reference \textbf{o}ntology, in a given extended development graph. We say $O$ \textit{heterogenously (conservatively) integrates} $O_1$ and $O_2$ if there are (conservative) refinements from both $O_1$ and $O_2$.
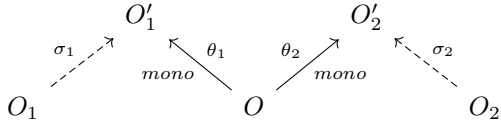
\[\begin{tikzcd}
	& {O_1'} && {O_2'} \\
	{O_1} && O && {O_2}
	\arrow["{\sigma_1}", dashed, from=2-1, to=1-2]
	\arrow["{\theta_1}"', "mono", from=2-3, to=1-2]
	\arrow["{\theta_2}", "mono"', from=2-3, to=1-4]
	\arrow["{\sigma_2}"', dashed, from=2-5, to=1-4]
\end{tikzcd}\]
    In this diagram, $O$ is the result of a heterogeneous integration between $O_1$ and $O_2$. If $\theta_1$ and $\theta_2$ are monomorphic, $O$ is the result of a conservative heterogeneous integration.
\end{defn}

\subsection{Connection}

Thus far, it appears that extended consequence systems and institutions are operationally analogous, as evidenced by the superficial similarity of the machinery for connecting and integrating. From this point onward, the operational parallels will be broken down by the introduction of a new tooling system. As a brief reminder, the objective of the connection operation between two \textbf{o}ntologies is to generate a third, novel \textbf{o}ntology that retains characteristics of the input \textbf{o}ntologies. In the da Costa-Tarskian approach, \textbf{o}ntologies may be connected by means of \textit{fibring}.

Fibring was originally presented by \citeauthor{Gabbay1996} as a means to combine two normal modal logics into a third normal bimodal logic \cite{Gabbay1996, Gabbay1996-GABFSA}. \textcite{10.1093/logcom/9.2.149} generalized fibring to deal with propositional logics in general --- given two propositional logics $L_1$ and $L_2$, the fibring of them is a new logic $L_1 \circ L_2$ which is a minimal conservative extension of both. Although fibring has been further generalized in various accounts (see, for instance, \cite{Sernadas2009-SERGFO, 7bb21a0a-8e04-322f-a004-4914849e5adf, caleiro-ramos-2007}), we will present a slightly modified version of algebraic fibring as presented in \cite[p. 150-160]{carnielli2008analysis} for extended consequence systems.

The intuitive idea behind fibring is it joins two logics by translating them into a language guaranteed to be conflict-free. As a matter of fact, before formally defining the fibring operation it is necessary to define translations between languages of consequence systems.

\begin{defn}[Translations and substitutions]
    Let $C$ and $C'$ be two signatures such that $C \leq C'$ and $g : L(C') \rightarrow \mathbb{N}$ a bijection. The \textit{translation}
    $$ \tau_g : L(C') \rightarrow L(C) $$
    is the map defined inductively as follows:
    \begin{itemize}
        \item $\tau_g(\xi_i) = \xi_{2i+1}$ for $\xi_i \in \Xi$
        \item $\tau_g(c) = c$ for $c \in C_0$
        \item $\tau_g(c') = \xi_{2g(c')}$ for $c' \in C'_0 \setminus C_0$
        \item $\tau_g(c(\gamma'_1,\ldots,\gamma'_k)) = (c(\tau_g(\gamma'_1),\ldots,\tau_g(\gamma'_k)))$ for $c \in C_k$, $k > 0$ and $\gamma'_1\,\ldots,\gamma'_k \in L(C')$
        \item $\tau_g(c'(\gamma'_1, \ldots, \gamma'_k)) = \xi_{2g(c'(\gamma'_1,\ldots,\gamma'_k))}$ for $c' \in C'_k \setminus C_k$, $k > 0$ and formulas $\gamma'_1,\ldots,\gamma'_k \in L(C')$
    \end{itemize}
    The preliminary \textit{substitution}
    $$ \tau_g^{-1} : \Xi \rightarrow L(C') $$
    is defined as
    \begin{itemize}
        \item $\tau_g^{-1}(\xi_{2i+1}) = \xi_i$ for $\xi_i \in \Xi$
        \item $\tau_g^{-1}(\xi_{2i}) = g^{-1}(i)$
    \end{itemize}
    If a preliminary substitution $\tau_g^{-1}$ is extended to $L(C)$, corresponding to the proper inverse of $\tau_g$, it is called a proper substitution or simply a substitution\footnote{In order to extend a preliminary substitution to a proper substitution, it is necessary to present the inductive rules for the values of $\tau_g^{-1}$ and prove that $\tau_g \circ \tau_g^{-1} = \tau_g \circ \tau_g^{-1} = id$. The details of this construction are present in \cite[p. 150-151]{carnielli2008analysis}.}.
\end{defn}

Given a variable in $\tau_g(L(C'))$, one can determine if it comes from a variable or a formula starting with a connective in $C'\setminus C$ if its index is even or odd. Additionally, it may be noted that the translation $\tau_g$ maps symbols of $C'$ that are not in $C$ to propositional variables. This is by design, to ensure that no symbols of $C'$ get mapped to existing symbols of $C$.

In order to define fibring of consequence systems $\langle C', \mathsf{C}' \rangle$ and $\langle C', \mathsf{C}'' \rangle$, it is necessary to define translations between $L(C' \cup C'')$ and the input languages, $L(C')$ and $L(C'')$. Given a bijection $g : L(C \cup C') \rightarrow \mathbb{N}$, such translations are defined as follows:

$$ \tau'_g : L(C' \cup C'') \rightarrow L(C') \text{ and } \tau''_g : L(C' \cup C'') \rightarrow L(C'') $$

The corresponding substitutions are
$$ \tau'^{-1}_g : L(C') \rightarrow L(C' \cup C'') \text{ and } \tau''^{-1}_g : L(C'') \rightarrow L(C' \cup C'') $$

In what follows, for the sake of readability, a few provisional measures will be taken. The suffix $g$ will be removed for translations and substitutions by assuming that a fixed bijection $L(C' \cup C'') \rightarrow \mathbb{N}$ exists. Given a set $\Gamma \in L(C)$, $\tau(\Gamma)$ denotes the following set:
$$\tau(\Gamma) = \{ \tau(\phi) : \phi \in \Gamma \}$$

Additionally, in order to minimize confusion between the signature of a consequence system and the consequence system itself, consequence maps will be denoted by $\vdash$ at times. In this case, given a consequence system $\langle C, \vdash \rangle$ and set $\Gamma \in L(C)$, $\Gamma^\vdash$ denotes $\vdash(\Gamma)$ and is called the \textit{deductive closure} of $\Gamma$.

Given two consequence systems $\mathfrak{C}'$ and $\mathfrak{C}''$, translations and substitutions define the closure of each $L \subseteq L(C' \cup C'')$ with respect to the consequence maps of each consequence system.

\begin{defn}
    Let $\mathfrak{C}' = \langle C', \vdash' \rangle$ and $\mathfrak{C}'' = \langle C'', \vdash'' \rangle$ be two consequence systems and let $\Gamma \subseteq L(C' \cup C'')$. Assume that $\tau' : L(C) \rightarrow L(C')$ and $\tau'' : L(C') \rightarrow L(C)$ denote the translations between the two consequence systems. Similarly, $\tau^{'-1}$ and $\tau^{''-1}$ denote the respective substitutions between the consequence systems. The $\vdash'$-\emph{closure} of $\Gamma$ is the set:

    $$ \Gamma^{\vdash'} = \tau^{'-1}((\tau'(\Gamma))^{\vdash'}) $$

    Similarly, the $\vdash''$-\emph{closure} of $\Gamma$ is the set $ \Gamma^{\vdash''} = \tau^{''-1}((\tau^{''}(\Gamma))^{\vdash''}) $
\end{defn}

Intuitively, the $\vdash'$-closure of a set $\Gamma \in L(C' \cup C'')$ denotes the set obtained by first translating $\Gamma$ to the language of consequence system $\mathfrak{C}'$, computing its deductive closure, and mapping it back to the language $L(C' \cup C'')$.

Finally, the fibring of consequence systems is defined as follows.

\begin{defn}
    Let $\mathfrak{C}' = \langle C', \vdash' \rangle$ and $\mathfrak{C}'' = \langle C'', \vdash'' \rangle$ be two consequence systems. The \emph{fibring of two consequence systems} $\mathcal{C}'$ and $\mathcal{C}''$ is a pair
    $$ \mathcal{C}'\cup\mathcal{C}'' = \langle C, \vdash \rangle $$
    where
    \begin{itemize}
        \item $C$ is a signature such that $C_k = C_k' \cup C_k''$ for every $k \in \mathbb{N}$
        \item $\vdash : \mathcal{P}(L(C)) \rightarrow \mathcal{P}(L(C))$ where, for each $\Gamma \subseteq L(C)$, $\Gamma^\vdash$ is inductively defined as follows:
        \begin{enumerate}
            \item $\Gamma \subseteq \Gamma^{\vdash}$
            \item If $\Delta \subseteq \Gamma^{\vdash}$, then $\Delta^{\vdash'} \cup \Delta^{\vdash''} \subseteq \Gamma^{\vdash}$
        \end{enumerate}
    \end{itemize}
\end{defn}

The object generated by fibring two consequence systems is also a consequence system, as proven by By \textcite[prop. 4.1.24]{carnielli2008analysis}. One may conceptualize the resulting consequence system as being a disjoint union of the two input consequence systems, where their respective consequence maps are preserved, albeit under translations and substitutions, via the deductive closure.

It is important to note that consequence systems are not \textbf{o}ntologies in the current context. Consequently, it is necessary to extend the definition of fibring to extended consequence systems. This will provide a means to in fact to connect \textbf{o}ntologies.

\begin{defn}[Heterogeneous connection]
    Let $O_1 = \langle C^1, \vdash^1, C^1_o, \Gamma^1_o \rangle$ and $O_2 = \langle C^2, \vdash^2, C^2_o, \Gamma^2_o \rangle$ be two extended consequence systems, i.e. \textbf{o}ntologies. A tuple $O_1 \cup O_2 = \langle C, \vdash, C_o, \Gamma_o \rangle$ is the \textit{connection} of $O_1$ and $O_2$ if:
    \begin{itemize}
        \item $\langle C, \vdash \rangle$ is the result of fibring between $\langle C^1, \vdash^1 \rangle$ and $\langle C^2, \vdash^2 \rangle$
        \item $C_o = C_o^1 \cup C_o^2$
        \item $\tau_1^{-1}(\phi) \cup \tau^{-1}_2(\psi) \subseteq \Gamma_o$ for each $\phi \in \Gamma^1_o, \psi \in \Gamma^2_o$
    \end{itemize}
    where $\tau_1^{-1}$ and $\tau_2^{-1}$ are the substitutions derived via fibring.
\end{defn}

The connection between two \textbf{o}ntologies constructs an object whose consequence system aspect is the result of the fibring of the consequence systems in the input \textbf{o}ntologies, and whose ontological aspect is preserved under substitution. We shall now prove the resulting object is indeed an \textbf{o}ntology.

\begin{prop}
    \label{het-connection-is-ontology}
    The heterogeneous connection $O_1 \cup O_2$ of \textbf{o}ntologies $O_1$ and $O_2$ is an \textbf{o}ntology.
\end{prop}
% \marginpar{Review.}
\begin{proof}
    Let $O_1 = \langle C^1, \vdash^1, C^1_o, \Gamma^1_o \rangle$ and $O_2 = \langle C^2, \vdash^2, C^2_o, \Gamma^2_o \rangle$. Suppose $O_1 \cup O_2 = \langle C, \vdash, C_o, \Gamma_o \rangle$. By \textcite[prop. 4.1.24]{carnielli2008analysis}, $\langle C, \vdash \rangle$ is a consequence system. By definition, $C_o$ is a signature and it is straight-forward to see that $C_o \subseteq C$.

    Consider $\emptyset^\vdash$ obtained via fibring, i.e the axiomatic theory of $\langle C, \vdash \rangle$. Clearly, by definition, we have that $\emptyset \in \emptyset^\vdash$. Therefore $\emptyset^{\vdash^1} \cup \emptyset^{\vdash^2} \subseteq \emptyset^\vdash$.
    
    Note that for each $\phi \in \Gamma^1_o$, $\tau^{-1}_1(\phi) \in \emptyset^{\vdash^1} \subseteq \emptyset^\vdash$. Similarly, for every $\psi \in \Gamma^2_o$, $\tau^{-1}_2(\psi) \in \emptyset^{\vdash^2} \subseteq \emptyset^\vdash$. Thus, $O_1 \cup O_2$ conserves the ontological aspect and therefore is an \textbf{o}ntology.
\end{proof}

The ontological aspect of the resulting connected \textbf{o}ntology naively merges the ontological aspects of the input \textbf{o}ntologies. That is, no additional steps are taken to handle the aforementioned issues of synonymy (a situation where different symbols have the same intended meaning) and homonymy (a situation where the same symbol has different meanins). Rather, these issues can be dealt with via morphisms in $\mathbf{ECsy}$. Recall the fact \textbf{o}ntologies may live inside an extended development graph, then heteregeneous connection can be interpreted as adding the coproduct of $O_1$ and $O_2$ into the graph alongside the required morphisms, as visualized in figure \ref*{fig:het-connection}.

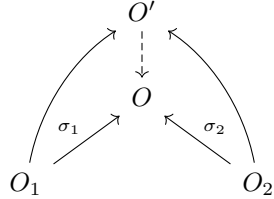
\begin{figure}[ht]
    \[\begin{tikzcd}
    & {O'} \\
    & O \\
    {O_1} && {O_2}
    \arrow["{\sigma_1}", from=3-1, to=2-2]
    \arrow["{\sigma_2}"', from=3-3, to=2-2]
    \arrow[curve={height=-12pt}, from=3-1, to=1-2]
    \arrow[curve={height=12pt}, from=3-3, to=1-2]
    \arrow[dashed, from=1-2, to=2-2]
    \end{tikzcd}\]
    \caption{Heterogeneous connection of $O_1$ and $O_2$.}
    \label{fig:het-connection}
\end{figure}

If two \textbf{o}ntologies share the same common vocabulary, i.e. the signature of the ontological aspect shares common symbols, then one may add morphisms to ``prepare'' the input \textbf{o}ntologies. This sort of ``preparation'' via extra morphisms allows one to ensure the input \textbf{o}ntologies refer to the same concepts using the same symbols and use different symbols when there is a homonymy issue. This is depicted in figure \ref*{fig:het-connection-preparation}.

\begin{figure}[h]
    \[\begin{tikzcd}
        &&& {O'} \\
        &&& O \\
        {O_1} & \ldots & {O_1^f} && {O_2^f} & \ldots & {O_2}
        \arrow[from=3-3, to=2-4]
        \arrow[dashed, from=1-4, to=2-4]
        \arrow[from=3-5, to=2-4]
        \arrow[curve={height=6pt}, from=3-5, to=1-4]
        \arrow[curve={height=-6pt}, from=3-3, to=1-4]
        \arrow[dashed, from=3-1, to=3-2]
        \arrow[dashed, from=3-2, to=3-3]
        \arrow[dashed, from=3-6, to=3-5]
        \arrow[dashed, from=3-7, to=3-6]
    \end{tikzcd}\]
    \caption{Heterogeneous connection of $O_1$ and $O_2$, with ``preparation''. Note the original input \textbf{o}ntologies may pass through many steps prior to connecting, resulting in final \textbf{o}ntologies $O_1^f$ and $O_2^f$.}
    \label{fig:het-connection-preparation}
\end{figure}
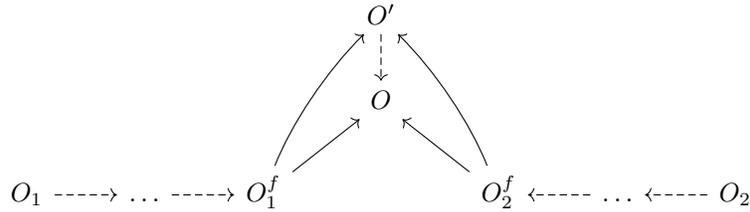

Ontological connection via fibring is a compelling concept, as it offers a means of combining ontologies while ensuring a minimal extension and preserving certain properties. Results compiled by \textcite[p. 150-160]{carnielli2008analysis} include preservation of structurality, weakness relation transitivity when the consequence systems of input \textbf{o}ntologies are structural and compactness. Additionally, they present three characteristics which increase the attractiveness of fibring: firstly, it consists of a homogeneous combination mechanism; secondly, the mechanism is algorithmic and easily extendable, as it has been done for extended consequence systems; thirdly, the resulting combination is canonical in that it is minimally stronger than the input consequence systems of the \textbf{o}ntologies.

Fibring-based ontological connection has a few clear shortcomings. As shown by \citeauthor{Marcelino2015} in a series of papers \cite{Marcelino2015, 10.1093/logcom/exw023, Marcelino2016-MARDAC-12,math10193481}, fibring in general does not preserve logical decidability --- a very sought-after property for computational purposes. However, disjoint fibring (when the input consequence systems have disjoint signatures) does preserve decidability. Additionally, \textcite{Marcelino2015} proved that, when decidability is preserved in the fibred consequence system, the complexity of the decision problem can be polynomially reduced to the worst decision problem of the input consequence systems.

As previously stated, consequence systems may be enriched with interpretation systems to construct signature-bound \textit{logic systems}. In addition, logic systems may also be fibred, which raises the question of the status of soundness and completeness preservation. \textcite{0e09d9fc-b459-37b1-b0d4-83462febd50c, 10.1093/jigpal/10.4.413} investigate under which conditions fibring preserves completeness, by showing classes of complete logic systems whose fibring is also complete.

Fibring, more specifically algebraic fibring, is far from the only kind of combination operation in the literature applicable to consequence systems, logic systems and logics in general. It is not in scope of this section to discuss in great detail about other operations, as its purpose is to establish an initial ground upon which to build heterogeneous \textbf{o}ntologies based on consequence systems. However, what follows is a very brief overview of a few of these works, as they address property-preservation issues and also illustrate potential paths for further development of the da Costian-Tarskianist approach:

\begin{enumerate}
    \item \textit{Graph-theoretic fibring}: \textcite{sernadas2009graph, Sernadas2009-SERGFO} present a graph-theoretic account of fibring based on multi-graphs (or \textit{m-graphs}) targetting the point-wise combination of the models of input logic systems. It has been shown by \textcite{10.1093/logcom/exq022} that this kind of fibring preserves the finite model property, which entails decidability under certain conditions (see \cite{Libkin2004} for more on this matter).
    \item \textit{Importing logics}: an asymmetric combination technique originally devised by \textcite{Rasga2012ImportingL} where the combined logic system is endowed with an importing connective, which allows formulas of the imported logic to be transformed into formulas of the importing logic. \textcite{3cf1a544-136a-3ddc-9d97-d36822e7c42f} have shown this technique does preserve soundness and completeness and in \cite{b61a702e-e512-3718-b3f0-487a04f93593} it is shown fibring may be characterized by \textit{biporting}, an extension of importing which includes an exporting connective.
    \item \textit{Meet-combination of logics}: originally devised by \textcite{8133014}, meet-combination is an approach to connecting logics based on combined constructors, which generate melded connectives inheriting the intersection of the properties of the input logic instead of their conjunction. Meet-combination was shown to preserve completeness and decidability under a set of admissable rules in \cite{RASGA_SERNADAS_SERNADAS_2016}.
\end{enumerate}

\subsection{Decomposition}

The intuitive idea of decomposition is to extract one or more sub-\textbf{o}ntologies from a given \textbf{o}ntology, in such a way that the sub-\textbf{o}ntologies are weaker, in a sense, than the original \textbf{o}ntology. The concept of ontological decomposition requires extending the definitions of \textbf{o}ntology and development graph to include the splitting machinery presented in \textcite[p. 391-400]{carnielli2008analysis}.

The initial step towards decomposition is the notion of a $k$-restricted language.

\begin{defn}[$k$-restricted language]
    Given $k \in \mathbb{N}$ and a signature $C$, $L(C)[k]$ is the set of formulas $\phi$ in $L(C)$ such that the set of schema variables occurring in $\phi$ is exactly $\{ \xi_1, \ldots, \xi_k \}$.
\end{defn}

The category $\mathbf{sSig}$ which extends $\mathbf{Sig}$ by introducing splitting signature morphisms, based on $k$-restricted languages.

\begin{defn}[Splitting signature morphism]
    Given two signatures $C$ and $C'$, a \emph{splitting signature morphism} $f$ between them, denoted $f : C \rightarrow C'$, is a mapping $f : L(C) \rightarrow L(C')$ such that if $c \in C_k$ then $f(c) \in L(C')[k]$.

    A splitting signature morphism $f$ induces a mapping $\hat{f} : L(C) \rightarrow L(C')$ such that:
    \begin{itemize}
        \item $\hat{f}(\xi) = \xi$ if $\xi \in \Xi$;
        \item $\hat{f}(c) = f(c)$ if $c \in C_0$;
        \item $\hat{f}(c(\phi_1, \ldots, \phi_k)) = f(c)(\hat{f}(\phi_1),\ldots,\hat{f}(\phi_k))$ if $c \in C_k$, $\phi_1, \ldots, \phi_k \in L(C)$ and $k > 0$.
    \end{itemize}
\end{defn}

The two categories $\mathbf{sSig}$, the category of splitting signatures, and $\mathbf{sCon}$, the category of splitting consequence systems, provide the necessary basis for ontological decomposition.

\begin{defn}[Category $\mathbf{sSig}$]
    The category $\mathbf{sSig}$ is defined as follows:
    \begin{itemize}
        \item Its objects are signatures;
        \item Its morphisms are splitting signature morphisms.
    \end{itemize}
    Composition between two signature morphisms $f : C \rightarrow C'$ and $g: C' \rightarrow C''$, denoted by $g \cdot f$, is defined to be the signature morphism $g \cdot f : C \rightarrow C''$ given by the mapping $\hat{g} \circ f: |C| \rightarrow L(C'')$.
    
    The identity morphism $id_C : C \rightarrow C$ is defined as $id_C(c) = c$ for $c \in C_0$ and $id_C(c) = c(\xi_1, \ldots, x_k)$ for $c \in C_k$, $k > 0$.
\end{defn}

% define sCon

\begin{defn}[Category $\mathbf{sCon}$]
    The category $\mathbf{sCon}$ of splitting consequence systems is defined as follows:
    \begin{itemize}
        \item Its objects are consequence systems;
        \item A morphism $f : \mathcal{C} \rightarrow \mathcal{C}'$ is a morphism $f : C \rightarrow C'$ in $\mathbf{sSig}$ such that, for every $\Gamma \cup \{\phi\} \subseteq L(C)$,
        $$ \Gamma \vdash \phi \text{ implies } \hat{f}(\Gamma) \vdash' \hat{f}(\phi)$$
    \end{itemize}
    Composition and identity morphisms are as in $\mathbf{sSig}$.
\end{defn}

It is not necessary to extend the $\mathbf{sCon}$ into an analogous category whose objects are extended consequence systems for a few reasons. The morphisms in $\mathbf{sCon}$, by definition, preserve entailment over a splitting signature morphism, which is necessary requirement in the context of ontological knowledge. Furthermore, the operations of ontological refinement, integration and connection insist on preserving or adding onto the ontological aspect of extended consequence systems --- it does not make sense to impose this restriction in the case of ontological decomposition, since the intuitive goal is to be able to extract significant chunks of an \textbf{o}ntology into  sub-\textbf{o}ntologies.

Before formally defining ontological decomposition, the following proposition is crucial to guarantee the machinery we use does indeed produce \textbf{o}ntologies.

\begin{prop}
    \label{prop:decomposition-structural}
    The category $\mathbf{sCon}$ has products of arbitrary small, non-empty families of objects. Moreover, if every object of the family is structural, so is the product.
\end{prop}
\begin{proof}
    Due to \textcite[prop. 9.1.6]{carnielli2008analysis}.
\end{proof}

In order to define ontological decomposition, one needs to extend the extended development graph to include extra links tied to splitting consequence system morphisms. This new structure is called an extended splittable development graph, henceforth ESDG, and it is defined as follows.

\begin{defn}[extended splittable development graph]
    An \emph{extended splittable development graph} or ESDG is an extended development graph $\mathcal{AG} = \langle \mathcal{N}, \mathcal{L} \rangle$ such that $\mathcal{L}$ contains additional directed links:
    \begin{itemize}
        \item \textit{splitting links} $K \xrightarrow{f} N$ for each $K, N$ such that there exists a morphism $f$ in $\mathbf{sCon}$ between the underlying consequence systems labelling $K$ and $N$
    \end{itemize}
\end{defn}

Finally, by adding splitting links, one can define ontological decomposition.

\begin{defn}[Ontological decomposition]
    Let $O = \langle C, \vdash, C_o, \Gamma_o \rangle$ be an \textbf{o}ntology in an ESDG. We say $O$ \textit{decomposes into a family of \textbf{o}ntologies} $O_1, \ldots, O_n$ if:
    \begin{enumerate}
        \item there are splitting links from $O$ to each $O_i$, $1 \leq i \leq n$;
        \item the generated diagram is a product.
    \end{enumerate}
    
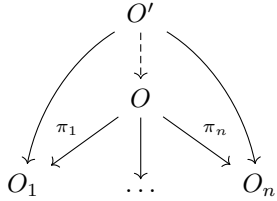
\[\begin{tikzcd}
        & {O'} \\
        & O \\
        {O_1} & \ldots & {O_n}
        \arrow["{\pi_1}"', from=2-2, to=3-1]
        \arrow["{\pi_n}", from=2-2, to=3-3]
        \arrow[curve={height=12pt}, from=1-2, to=3-1]
        \arrow[curve={height=-12pt}, from=1-2, to=3-3]
        \arrow[dashed, from=1-2, to=2-2]
        \arrow[from=2-2, to=3-2]
        \end{tikzcd}\]
\end{defn}

Observe that, as previously stated, decomposition does not impose restrictions on the ontological aspect and instead relies on entailment preservation provided by splitting morphisms. Perhaps non-intuitively, decomposition also does not enforce the ``component'' \textbf{o}ntologies to be weaker than the input \textbf{o}ntology in any sense. It is often times desirable to decompose into a weaker family of \textbf{o}ntologies, however it is not necessary to ensure at minimum structurality as per proposition \ref{prop:decomposition-structural}.

Ontological decomposition is also a means to provide possible-translation characterization and semantics to \textbf{o}ntologies, in the sense of \textcite{carnielli2008analysis}. As the current objective is to present an initial proposal of ontological decomposition, it is not in scope to address this topic in depth at this time. Nevertheless, it is possible that ontological decomposition could result in the development of interesting ontological machinery. Indeed, some \textbf{o}ntologies, in practice, already implement some form of decomposition as a means of describing a target \textbf{o}ntology. For example, e.g. BFO is decomposed into its SNAP and SPAN sub-\textbf{o}ntologies.

\section{Conclusions and future work}

At this stage, da Costian-Tarskianism is merely a proposal, for there is no practical, ready-to-use, production-ready implementation of it similar to Carnapian-Goguenism's \textsc{Hets}. Notwithstanding, it is possible to implement the approach using a computational category-theoretical framework, such as the one currently being developed by the AlgebraicJulia\footnote{\href{https://github.com/AlgebraicJulia}{https://github.com/AlgebraicJulia}} project, which in turn is based on the work developed by \textcite{Patterson2022categoricaldata} on attributed $\mathcal{C}$-sets.

Despite lack of implementation, da Costian-Tarskianism has a few theoretical benefits over Carnapian-Goguenism:
\begin{itemize}
    \item The new, revised definition of \textbf{o}ntology is based on syntax presentation and is not restrictive of any kind of semantics, allowing one to attach novel and non-relational semantics to \textbf{o}ntologies.
    \item Fibring-based ontological connection does not force generated \textbf{o}ntologies to have certain properties regardless of input \textbf{o}ntologies. Furthermore, more complex combination operations historically derived from fibring can be incorporated into da Costian-Tarskianism to ensure ontological connection preserves additional properties.
    \item Adding a new logic to an extended development graph does not require rebuilding an entire artificial construct, such as the Grothendieck institutions, since extended consequence systems are structurally simple and abstract enough to not require flattening.
    \item da Costian-Tarskianism supports one extra heterogeneous operation, namely ontological decomposition. However, it is noteworthy that Carnapian-Goguenism could also account for a similar operation using a categorical universal property construct, similarly to what has been proposed in this section.
\end{itemize}

For those seeking a concise overview, table \ref{tab:cg-dt-comparison} presents a comparative analysis between Carnapian-Goguenism and da Costian-Tarskianism.

\begin{table}[h]
    \centering
    \label{tab:cg-dt-comparison}
    \begin{tabular}{ccc}
        & Carnapian-Goguenism & da Costian-Tarskianism \\
\hline
Structural underpinning & Semantics & Syntax \\
Philosophical underpinning & Descriptive & Normative \\
Refinement?     & \cmark   & \cmark             \\
Integration?    & \cmark   & \cmark             \\
Connection?     & \cmark   & \cmark             \\
Decomposition?  & \xmark   & \cmark             \\
Implementation maturity  & Production-ready & Baby-steps
\end{tabular}
\caption{Table summarizing key traits of Carnapian-Goguenism and da Costian-Tarskianism.}
\end{table}

\printbibliography

\end{document}